\documentclass[conference]{IEEEtran}
\IEEEoverridecommandlockouts
\usepackage{cite}
\usepackage{amsmath,amssymb,amsfonts}
\usepackage{graphicx}
\usepackage{dsfont}
\usepackage{textcomp}
\usepackage{xcolor}
\usepackage{amsthm}
\usepackage{algorithm}
\usepackage{algpseudocode}  

\usepackage{hyperref}

\newtheorem{theorem}{Theorem}
\def\BibTeX{{\rm B\kern-.05em{\sc i\kern-.025em b}\kern-.08em
    T\kern-.1667em\lower.7ex\hbox{E}\kern-.125emX}}
\begin{document}

\title{Sparse Interpretable Deep Learning with LIES Networks for Symbolic Regression\\
}

\author{
\IEEEauthorblockN{Mansooreh Montazerin}
\IEEEauthorblockA{
Department of Electrical \& Computer Engineering\\
University of Southern California\\
}
\and
\IEEEauthorblockN{Majd Al Aawar}
\IEEEauthorblockA{
Department of Electrical \& Computer Engineering\\
University of Southern California\\
}
\and
\IEEEauthorblockN{Antonio Ortega}
\IEEEauthorblockA{
Department of Electrical \& Computer Engineering\\
University of Southern California\\
}
\and
\IEEEauthorblockN{Ajitesh Srivastava}
\IEEEauthorblockA{
Department of Electrical \& Computer Engineering\\
University of Southern California\\
}
}


\maketitle
\IEEEpeerreviewmaketitle

\begin{abstract}
Symbolic regression (SR) aims to discover closed-form mathematical expressions that accurately describe data, offering interpretability and analytical insight beyond standard black-box models. Existing SR methods often rely on population-based search or autoregressive modeling, which struggle with scalability and symbolic consistency. We introduce LIES (Logarithm, Identity, Exponential, Sine), a fixed neural network architecture with interpretable primitive activations that are optimized to model symbolic expressions. We develop a framework to extract compact formulae from LIES networks by training with an appropriate oversampling strategy and a tailored loss function to promote sparsity and to prevent gradient instability. After training, it applies additional pruning strategies to further simplify the learned expressions into compact formulae. Our experiments on SR benchmarks show that the LIES framework consistently produces sparse and accurate symbolic formulae outperforming all baselines. We also demonstrate the importance of each design component through ablation studies. 
\end{abstract}

\begin{IEEEkeywords}
Symbolic Regression, Interpretable Machine Learning, Scientific Discovery, Neural Network Pruning
\end{IEEEkeywords}

\section{Introduction}\label{sec:intro}

Uncovering the underlying mathematical laws that govern complex systems is a fundamental goal in science and engineering. Symbolic Regression (SR) serves as a powerful approach to this task by discovering abstract mathematical expressions that best describe relationships within observed data and provide insights into the mechanism of the underlying processes~\cite{physics2023}. 
Unlike conventional regression techniques that fit parameters within predefined models, SR searches for both the structure and parameters of equations, offering a flexible and interpretable way to model intricate phenomena~\cite{marin2023}. This ability makes SR highly valuable across disciplines, enabling the derivation of governing equations in fields such as fluid mechanics, molecular interactions, astrophysics, and materials science~\cite{sr-astro,sr-pde,sr-dynamic,sr-material}.

A key advantage of SR is its interpretability. Resulting models are expressed in symbolic form, making them understandable and generalizable beyond the training data~\cite{cranmer-interp}. These models can be derived from experimental measurements, simulations, or real-world observations, allowing researchers to gain insights into the underlying mechanisms of physical systems~\cite{aifeynman}. 
However, discovering meaningful symbolic representations is inherently challenging due to the vast combinatorial search space of potential equations. Finding a balance between accuracy, simplicity, and generalization is crucial, as overly complex models risk overfitting while overly simplistic ones may fail to capture essential system dynamics~\cite{lacava}.

SR spans a wide range of methods from traditional approaches like Genetic Programming (GP) to advanced Deep Learning (DL) methods like Transformers, Graph Neural Networks (GNNs), and deep generative models~\cite{sr-gp2009,sr-gp2022,sr-gp2023,sr-trans2023,sr-trans2024,cranmer-interp,sr-gen2023}. Evolutionary algorithms like GP~\cite{sr-gp2023} explore the space of mathematical expressions by iteratively refining candidate solutions. These equations are constructed using fundamental components, known as primitives, which can include constants and elementary functions like addition, multiplication, and trigonometric operations. However, these methods use a predefined set of operations and construct candidate expressions through a sequential generation and evaluation process, which often leads to overly complex and less interpretable equations. 
Alternatively, recent DL methods~\cite{sr-trans2023,cranmer-interp,sr-gen2023} aim to generate candidate formulae by treating expressions as sequences, structured graphs, or samples from learned latent spaces, enabling better search for candidate expressions compared to traditional GP-based methods.
These approaches offer better scalability and adaptability but face several limitations: they often require large amounts of training data, struggle with generating syntactically valid or semantically meaningful expressions, and offer limited interpretability during training, as symbolic forms are only revealed after decoding.

To address these challenges, we propose a new framework that uses a neural network architecture while preserving the interpretability of traditional GP.
Specifically, it uses stacking layers of a small set of operators as activations in such a way that any candidate formula can be represented by sparsification and pruning the network. Therefore, it reduces the problem of learning the best candidate to the problem of learning compact and sparse neural networks.
Our model, LIES, is a feedforward architecture with a fixed set of activation functions --- bounded logarithm (L) and exponential (E), sine (S), and identity (I) --- carefully chosen to capture operations commonly found in natural and scientific laws, such as multiplication, division, exponentiation, and periodic behavior. Unlike GP-based methods that generate expressions through iterative search, LIES trains a fixed network whose structure and activation functions are explicitly aligned with symbolic operations. This design enables the model to recover interpretable expressions via structured pruning, while retaining the data efficiency and expressiveness of neural networks. In contrast to black-box DL models, LIES incorporates symbolic structure directly into the architecture, requiring less training data and producing more compact, meaningful formulae.
The overall framework progressively sparsifies the trained network to distill a final symbolic expression that is both concise and interpretable.

Specifically, our main contributions are as follows:

\begin{itemize}
    \item We propose a novel architecture that incorporates a specific set of activation functions (L, I, E, and S) reflecting operations common in natural laws to achieve interpretable symbolic structures. The model is trained with a modified loss function that helps avoid unstable gradients due to logarithm and exponential functions.
    
    \item We introduce a pipeline that promotes sparsity and interpretability through a sparsity-aware loss function, combined with model pruning and symbolic simplification. 
    \item We evaluate our method on 61 formulae from the AI Feynman dataset~\cite{aifeynman}, showing strong symbolic and numerical performance, and provide an ablation study to isolate the contribution of each component in the pipeline.
\end{itemize}


\section{Related Works}\label{sec:related}
Symbolic regression (SR) is a foundational technique for discovering underlying mathematical relationships from data. It aims to automatically generate symbolic expressions or closed-form equations that best explain the observed input–output behavior. Genetic programming (GP) plays a central role in SR, evolving populations of candidate expressions through iterative processes such as mutation, crossover, and selection. 
Pioneering work by John Koza~\cite{koza-genetic1994} laid the foundation for GP-based SR, evolving expressions through mimicry of biological selection. Subsequent works aimed to improve the efficiency of GP and its applicability to a wider range of data types. For example, Gustafson et al.~\cite{gustafson-imp2005} improved GP by preventing crossover between individuals with identical fitness, which often produced redundant offspring. This simple constraint led to better performance, especially on more complex SR tasks. 
PySR~\cite{sr-gp2023} represents a more recent and performant GP-based SR system. It introduces a multi-population evolutionary algorithm with an evolve–simplify–optimize loop, enabling the discovery of concise and interpretable expressions, especially for scientific datasets. Nonetheless, GP-based methods in general remain computationally expensive, sensitive to hyperparameter settings, and prone to generating overly complex or redundant expressions due to their heuristic nature.  

In contrast, deep learning (DL)-based SR approaches leverage the representation learning capabilities of neural networks to model symbolic mappings more flexibly. Petersen et al.~\cite{sr-dl2019} introduced a deep reinforcement learning approach to SR, where a Recurrent Neural Network generates expressions sequentially. Biggio et al.~\cite{biggio-seq2seq2020} proposed a seq2seq model mapping input-output pairs to a symbolic expression built from a vocabulary. SR can also be deemed as surrogate modeling, offering interpretable and tractable approximations for complex, expensive, or black-box functions. Gaussian process surrogates~\cite{gaussian2006} and Physics-Informed Neural Networks (PINNs)~\cite{pinn2022} are two prominent examples.

Despite these advances, both GP-based and DL-based methods face persistent challenges: GP methods often rely on inefficient, heuristic-driven search and yield expressions that are unnecessarily complex or poorly generalizing, while DL methods tend to operate as black boxes, requiring substantial data and offering limited symbolic interpretability or structure. We propose a novel DL-based framework that bridges the strengths of both paradigms. Rather than generating expressions explicitly or relying on symbolic decoding, we train a fixed feedforward neural network—called LIES—composed of activation functions selected to reflect symbolic operations found in natural laws. This architecture encodes symbolic structure directly into the network, enabling expression discovery through structured pruning and gradient-based sparsification. Our approach eliminates the need for evolutionary search, retains data efficiency, and produces compact, interpretable formulae.

EQL$\div$\cite{eql2018} employs a neural architecture with fixed activation functions designed to resemble symbolic operations, including both unary (e.g., identity, sine, cosine) and binary (e.g., multiplication, division) operators. 
Yet, this architecture has some constraints, such as division being restricted to the final layer, and the class of expressions it can represent is limited to rational combinations of polynomial and trigonometric functions. It provides no framework for obtaining compact formulae and struggles to learn expressions with many simple but ubiquitous operations like exponents, logarithms, and roots~\cite{sr-dl2019}.
In contrast, our network, composed solely of \textit{unary} operators—L, I, E, and S—acts as a universal approximator (see \autoref{subsec:arch}) that can flexibly represent a wider range of natural laws. We propose a carefully designed loss function and a framework that includes sparsification, pruning, and symbolic simplification, enabling accurate learning of a large class of compact formulae.


\section{Methodology}\label{sec:method}

\subsection{Preliminaries}\label{subsec:prel}

In data-driven equation discovery, commonly referred to as symbolic regression (SR), our objective is to find a compact and interpretable mathematical expression $\hat{f}$ that closely approximates an unknown target function \( f: \mathbb{R}^d \to \mathbb{R},\) which maps a 
$d$-dimensional input vector $\mathbf{x}$ to an output $y$. 
Given a dataset of $n$ samples \(D = \{(\mathbf{x}_i, y_i)\}_{i=1}^{n}\), SR tries to find the underlying mathematical relationship between the input-output pairs such that $\hat{f}(\mathbf{x}_i)\approx y_i$ for all data samples. Apart from fitting the observed data accurately, the discovered equation should be interpretable and capable of generalizing well to unseen inputs, ensuring its utility beyond the training dataset.

\subsection{The Proposed LIES Architecture}\label{subsec:arch}

Unlike prior approaches that generate mathematical expressions through population-based search or sequential modeling, we propose a fundamentally different strategy: We design a fixed neural network—called the LIES network—whose structure and activation functions are crafted to naturally simplify into symbolic expressions, and develop a training strategy to favor sparsity. Our primary hypothesis is that most laws of nature can be expressed using a small set of primitive functions, i.e., logarithm ($L$), identity ($I$), exponential ($E$), and sine ($S$), applied to inputs and constants. 
For instance, simple ubiquitous arithmetic operations, such as the multiplication and division of two quantities $a$ and $b$ are given by $\exp(\ln(a) + \ln(b))$ and $\exp(\ln(a) - \ln(b))$, respectively. Note that $\cos(a)$ can be written as $\sin(\pi/2 - a)$, and therefore, any trigonometric function can also be represented using the sine function. Other operations like inverse trigonometric functions, while rare in natural laws, can still be approximated by polynomials using the Taylor series. The proposed LIES architecture (represented in \autoref{fig:LIES}) is a multi-layer feed-forward network with exactly four neurons per layer, each corresponding to one of the primitive activation functions. An $L$-layer configuration consists of $L-1$ hidden layers for which there is a linear mapping followed by non-linear transformations. The output \( \mathbf{z}^{(i)} \) of the \( i^\text{th} \) layer can be represented as:
\begin{align}
    \mathbf{z}^{(i)} &= f\left(\mathbf{h}^{(i)}\right), \\
    \mathbf{h}^{(i)} &= \mathbf{W}^{(i)} \mathbf{z}^{(i-1)},
\end{align}
where $f$ is the non-linear activation function, $\mathbf{W}^{(i)}$ is the weight matrix, $\mathbf{h}^{(i)}$ is the vector of pre-activation units and $\mathbf{z}^{(0)}=\mathbf{x}$ is the input data. Since this is a regression task, the activation function for the final layer is deleted, and we have the output as ${\hat{y}} = \mathbf{W}^{(L)} \mathbf{z}^{(L-1)}$. 

Additionally, we include fully dense residual connections between all the layers of the network to improve gradient flow and help the network to learn a wider variety of mathematical expressions in the output. Here, we do not consider any bias terms for the layers, but we add a vector with all entries equal to $1$ in the input layer to act as the bias.

\begin{figure}[htbp]
\centering
  \includegraphics[trim=0 13 0 13, clip, width=0.65\linewidth]{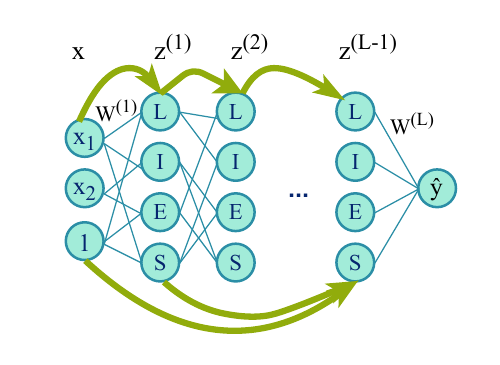}
  \caption{Architecture of the proposed LIES network}
  \label{fig:LIES}
\end{figure}

The expressive power of our proposed architecture is justified by the following theorem.

\begin{theorem}
\textit{LIES Networks are universal approximators.}
\end{theorem}

\begin{proof}
We establish the universality of the LIES network by showing that it can emulate any standard multilayer perceptron (MLP) with one or more hidden layers of arbitrary width. Specifically, \autoref{fig:sig}a illustrates how the LIES network can be configured to approximate a sigmoid activation function (Sigmoid block). \autoref{fig:sig}b further demonstrates how an $L$-layer MLP can be systematically transformed into an equivalent LIES network. Additionally, we can prove the universal approximability by showing, through a similar argument, that LIES networks can replicate any polynomial function.
\end{proof}

\begin{figure*}[htbp]
\centering
  \includegraphics[trim=0 0 0 0, clip, width=\linewidth]{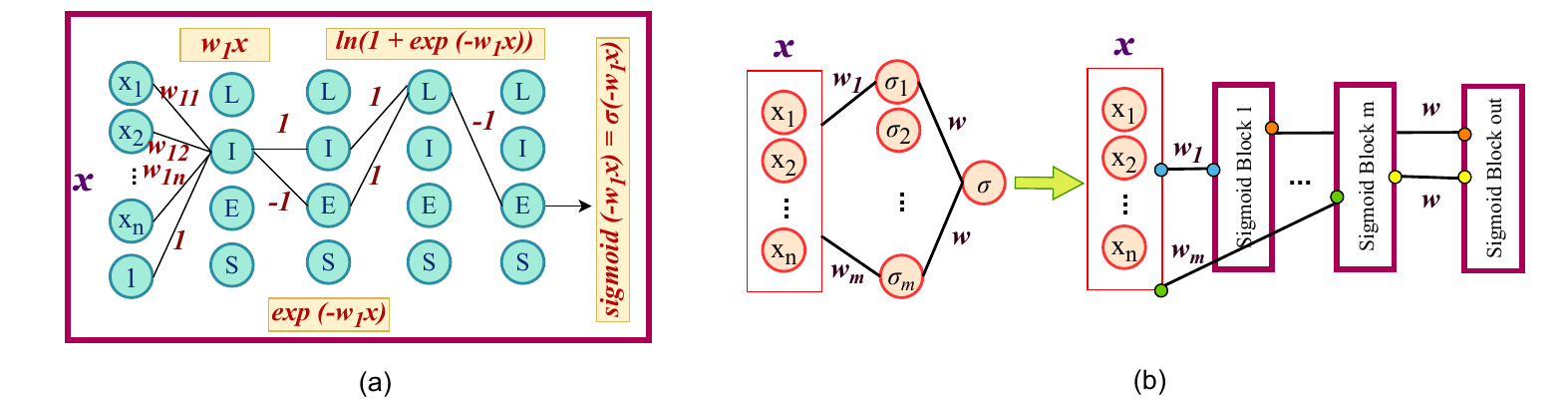}
  \caption{(a) Approximation of the Sigmoid function using a LIES-based configuration. (b) Transformation of an $L$-layer MLP into an equivalent LIES network using stacked LIES layers and Sigmoid blocks.}
  \label{fig:sig}
\end{figure*}

The overall architecture and processing pipeline of our proposed method are illustrated in ~\autoref{fig:pip}.

\begin{figure*}[htbp]
\centering
  \includegraphics[width=0.9\linewidth]{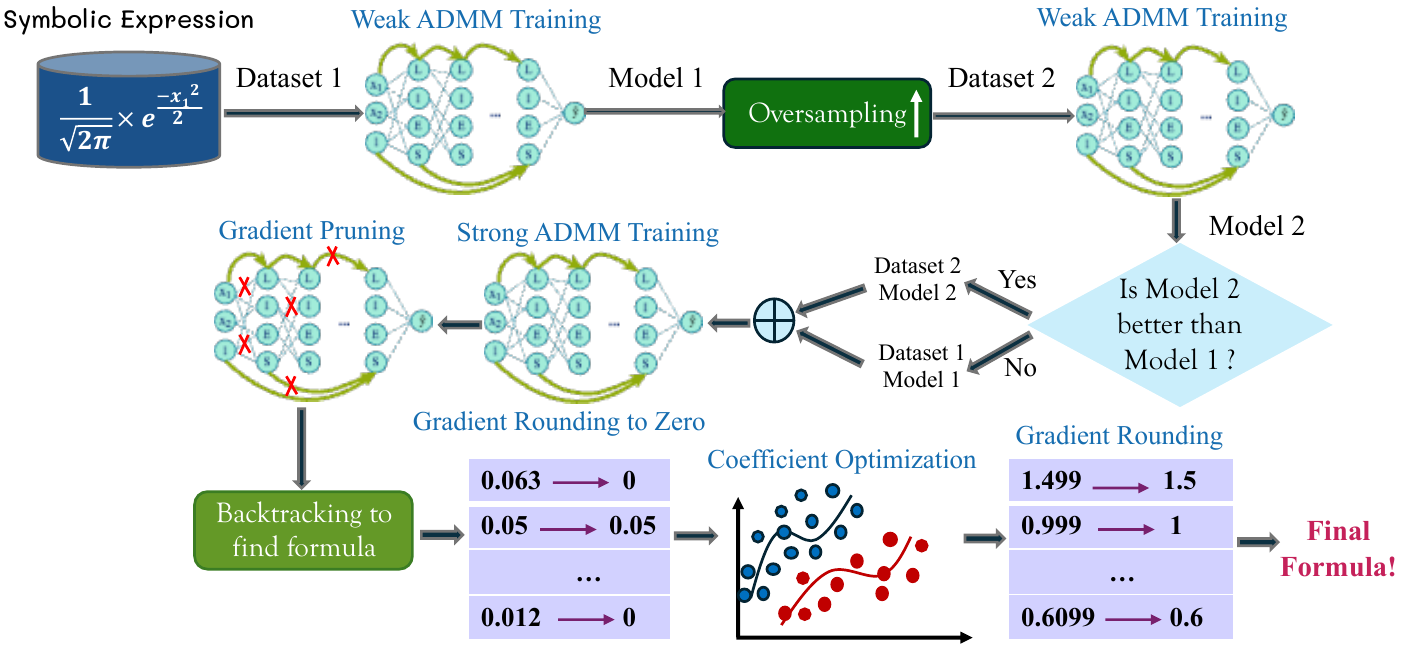}
  \caption{End-to-end pipeline of the proposed framework. The process starts with weak ADMM training on the original dataset, followed by oversampling to improve predictions in high-error regions. A second round of weak ADMM training is performed, followed by strong ADMM training and gradient-based pruning to enforce sparsity. After extracting the symbolic expression, gradient-based rounding to zero removes negligible coefficients, coefficient optimization refines the remaining constants, and a final rounding step ensures symbolic clarity by adjusting constants to cleaner representations.}
  \label{fig:pip}
\end{figure*}

\subsection{Sparsity}\label{subsec:sparsity}

Research on enforcing sparsity in neural networks is gaining more attention, as deep learning models with millions of parameters are becoming increasingly costly and difficult to compute~\cite{marin2023,kim-spars2020}. In the context of SR, to maintain the interpretability and generalizability of the model to unseen data, the system must be guided toward discovering the most concise mathematical expression that accurately represents the data. More specifically, in genetic programming-based methods, this can be achieved by restricting the total number of terms in the generated expression~\cite{cranmer-interp}. 
SINDy~\cite{sindy} enforces sparsity by first representing the system using a set of possible functions and then progressively removing the least important terms until only the most relevant ones remain. Sparsity in reinforcement learning-based symbolic regression is achieved by designing reward functions that penalize complexity, restricting available operators, and guiding the agent toward simpler expressions through policy learning and search constraints~\cite{rl-spars2023,rl-spars2025}. 
In the LIES network, we employ three complementary pruning strategies to promote sparsity: (1) an Alternating Direction Method of Multipliers (ADMM) optimization algorithm that formulates the pruning process as a constrained optimization problem and solves it iteratively to efficiently reduce model parameters while maintaining accuracy~\cite{admm2018}, (2) node pruning which targets the removal of entire neurons that remain active despite having pruned inputs (e.g., $exp(0) = 1$),  
and (3) a gradient-based pruning approach that evaluates the sensitivity of the network’s output to individual weights, pruning those whose removal leads to changes below a predefined threshold. 

(1) \textit{ADMM Weight Pruning}: Our training process formulates an 
ADMM pruning framework~\cite{boyd-admm} that promotes sparsity by penalizing the absolute values ($\ell_1$-norm) of the weights. We define the $\ell_1$-regularized loss minimization problem:
\begin{equation}
\min_W L(W) + \lambda \|W\|_1 
\end{equation}
where $L(W)$ is the original loss function of the network and $\lambda$ is the regularization coefficient controlling the strength of the sparsity penalty.
This problem is difficult to solve due to the non-smoothness of the $\ell_1$-term. 

To address this, we introduce an auxiliary variable $Z$, leading to the optimization:
\begin{equation}
\min_{W, Z} L(W) + \lambda \|Z\|_1 \quad \text{subject to} \quad W = Z, 
\end{equation}
whose augmented Lagrangian is given by:
\begin{equation}
\mathcal{L}_{\rho}(W, Z, U) = L(W) + \lambda \|Z\|_1 + \frac{\rho}{2} \|W - Z + U\|_2^2 - \frac{\rho}{2} \|U\|_2^2
\end{equation}
where $U$ is a scaled dual variable (Lagrange multiplier) to ensure that $W$ and $Z$ remain close, and $\rho$ is a penalty parameter controlling how strongly $W$ and $Z$ should match.

The ADMM algorithm alternates between updating $W$, $Z$, and $U$. For $W$, we use a gradient descent update minimizing the loss function while softly encouraging proximity to a sparse target, without directly imposing sparsity constraints: 
\begin{equation} \label{eq:wupdate}
W^{(k+1)} = \arg\min_W L(W) + \frac{\rho}{2} \|W - Z^{(k)} + U^{(k)}\|_2^2
\end{equation}
The $Z$-update step is a proximal operator update:
\begin{equation}
Z^{(k+1)} = \arg\min_Z \lambda \|Z\|_1 + \frac{\rho}{2} \|W^{(k+1)} - Z + U^{(k)}\|_2^2, 
\end{equation}
which is solved by applying soft thresholding to shrink small values to zero:
\begin{equation}
Z^{(k+1)} = \text{sign}(W^{(k+1)} + U^{(k)}) \cdot \max(|W^{(k+1)} + U^{(k)}| - \frac{\lambda}{\rho}, 0). 
\end{equation}
Finally, the $U$-update step accumulates differences between $W$ and $Z$, ensuring that they gradually converge.
\begin{equation}
U^{(k+1)} = U^{(k)} + (W^{(k+1)} - Z^{(k+1)})
\end{equation}

(2) \textit{Node Pruning}: While weight pruning effectively eliminates individual terms within activations, it does not always remove entire neurons. This can be particularly problematic for two of our primitive activations (logarithm (L) and exponential (E)) that produce non-zero outputs even when all input terms are set to zero. 
Consequently, such neurons can continue to propagate non-zero signals to subsequent layers as long as they retain any output connection. More specifically, the exponential unit with zero input yields an output of $1$, and the logarithmic unit with zero input---being below the cutoff threshold $x_l$---outputs $\log(x_l)$, rather than being suppressed entirely. To mitigate this issue, we introduce a node pruning mechanism designed to completely eliminate unused activations. Our approach reduces node pruning to a structured form of edge pruning by introducing auxiliary edges. Specifically, each neuron in the neural network is split into two sequential sub-neurons which have the original activation (i.e., L, E, S or I) followed by an identity (I) activation. 
Only a single auxiliary edge connects the original activation and the Identity activation neurons, and this edge is then connected to all subsequent layers (\autoref{fig:nodeprune}). We incorporate this into our ADMM formulation by imposing a constraint on the number of non-zero auxiliary edges, effectively controlling the number of active nodes in the network. 
When an auxiliary edge is pruned to zero, the corresponding original activation is completely removed from the LIES network. This enables the model to not only learn sparse combinations of terms but also to discard entire neurons that do not contribute meaningfully to the symbolic expression.

\begin{figure}[htbp]
\centering
  \includegraphics[trim=0 0 0 0, clip, width=1\linewidth]{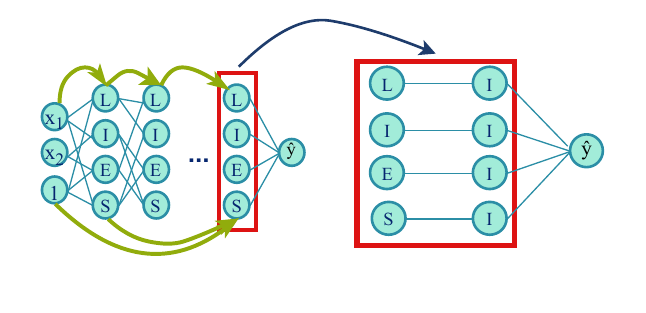}
  \caption{Configuration of the node pruning in LIES network.}
  \label{fig:nodeprune}
\end{figure}

(3) \textit{Gradient-based Pruning}: As a final pruning step, applied after weight and node pruning and once the model has been trained, we introduce a gradient-based mechanism to further refine the network structure. 
This approach leverages the intuition of first-order Taylor expansion, similar to prior sensitivity-based pruning methods~\cite{grad-prune2016,brain-damage}, but differs in that it approximates the change in the network’s output rather than the loss function when a parameter is set to zero. Theoretical justification is provided in~\autoref{thm:output-taylor}, which establishes a first-order bound on output variation resulting from parameter perturbation. This formulation allows us to quantify each parameter’s influence and prune those whose removal leads to negligible changes in output, thereby enhancing sparsity while preserving the functional fidelity of the symbolic expression. 

\begin{theorem}[Gradient-based Pruning/Rounding]
\label{thm:output-taylor}
Let \( f : \mathbb{R}^{p + d} \rightarrow \mathbb{R} \) be a function of \( p \) parameters and an input \( \mathbf{x} \in \mathbb{R}^d \) (representing the variables), with continuous partial derivatives with respect to a parameter \( P_i \). If
\[
\left\lvert h \cdot \frac{\partial f(P_i; \mathbf{x})}{\partial P_i} \right\rvert < \epsilon \quad \forall P_i \in [\alpha, \alpha + h], \, \mathbf{x},
\]
for some small \( \epsilon > 0 \), then
\[
f(P_i = \alpha + h; \mathbf{x}) \approx f(P_i = \alpha; \mathbf{x}).
\]

\end{theorem}

\begin{proof}
Let $g_{\mathbf{x}}(y) = f(P_i = y; \mathbf{x})$.
Then, by the mean value theorem, there exists \( \alpha_0 \in [\alpha, \alpha+h] \) such that
\[
g_{\mathbf{x}}(\alpha+h) - g_{\mathbf{x}}(\alpha) = h \cdot g_{\mathbf{x}}'(\alpha_0).
\]
As a result, we have the bound
\[
|g_{\mathbf{x}}(\alpha+h) - g_{\mathbf{x}}(\alpha)| = |h \cdot g_{\mathbf{x}}'(\alpha_0)| \leq h \cdot \max_{y \in [\alpha, \alpha+h]} |g_{\mathbf{x}}'(y)|,
\]
where the maximum exists due to the continuity of \( g_{\mathbf{x}}' \) on the closed interval.
If this upper bound is smaller than a predefined threshold \( \tau \), i.e.,
\[
h \cdot \max_{y \in [\alpha, \alpha+h]} |g_{\mathbf{x}}'(y)| < \tau,
\]
then the change in the function value is considered negligible, and we may approximate:
\[
g_{\mathbf{x}}(\alpha+h)  \approx g_{\mathbf{x}}(\alpha) .
\]
Moreover, if the above condition holds for all \( \mathbf{x} \in S \), where \( S \) denotes the set of all input samples to the network, then we approximate:
\[
f(\alpha + h) \approx f(\alpha).
\]
\end{proof}

To quantify the effect of individual parameters, let \( w_i \) denote a weight in the trained model. We define its importance as the change in the network’s output when it is set to zero for a given input $\mathbf{x}$:
\[
|\Delta O(w_i, \mathbf{x})| = |O(w_i = 0, \mathbf{x}) - O(w_i, \mathbf{x})|
\]
Here, \( O(w_i, \mathbf{x}) \) denotes the output of the network when \( w_i \) is active, and \( O(w_i = 0, \mathbf{x}) \) denotes the output when \( w_i \) is pruned.
By \autoref{thm:output-taylor}, this change is bounded as:
\[
|\Delta O(w_i, \mathbf{x})| \leq |w_i| \cdot \max_{s \in [0, w_i]} \left| \frac{\partial O(s, \mathbf{x})}{\partial w_i} \right|.
\]
Since the gradient \( \frac{\partial O}{\partial w_i}(w_i, \mathbf{x}) \) is already computed during backpropagation, we avoid additional evaluations by approximating the maximum with the gradient at the current weight value:
\[
\left| \frac{\partial O}{\partial w_i}(s, \mathbf{x}) \right| \approx \left| \frac{\partial O}{\partial w_i}(w_i, \mathbf{x}) \right|.
\]
Thus, if the product \( |w_i \cdot \frac{\partial O}{\partial w_i}(w_i, \mathbf{x})| \) is smaller than a predefined threshold for all input samples $\mathbf{x}$, we prune \( w_i \) as its influence on the output is considered negligible.

\subsection{Oversampling}

To improve symbolic recovery and predictive accuracy, our methodology incorporates an oversampling step that targets regions of the input space where the model exhibits the largest prediction errors. By introducing additional training samples in these high-error regions, the model is encouraged to refine its predictions and generalize more effectively. 
The oversampling procedure is thoroughly explained in~\autoref{alg:oversampling} where we set \texttt{max\_iter}$=2$, $P=30$ and $k=8$. 
The training process consists of two rounds of weak ADMM training—one before and one after the oversampling phase—followed by a final round of strong ADMM training. Symbolic expressions are then extracted for evaluation, with oversampling contributing to higher $R^2$ scores and more accurate symbolic solutions (explained in \autoref{subsec:metrics}).

\begin{algorithm}
\caption{Oversampling}
\label{alg:oversampling}
\begin{algorithmic}[1]
\State Initialize dataset $\mathcal{D}_1$ with $n$ input variables
\State Model $M_{1} \gets \text{weakADMM}(\mathcal{D}_{1})$
\State Divide the range of each variable into $k$ equal-width bins
\State Construct the $k^n$ grid over the input space
\For{each bin $b \in \{1, \dots, k\}^n$}
    \State Error$[b] \gets 1 - R^2$ of $M_1$ in $b$
\EndFor
\State $E_1 \gets $ mean(Error)
\For{$t = 1$ to \texttt{max\_iter}}
    \State $b^* \gets \arg\max $ Error$(b)$
    \State $\mathcal{D}_{t+1} \gets \mathcal{D}_{t} \cup P \%  \text{ additional points in } b^*$
    \State $M_{t+1} \gets \text{weakADMM}(\mathcal{D}_{t+1})$
    \State Compute total error $E_{t+1}$ using $M_{t+1}$ 
    \If{$E_{t+1} \geq E_t$}
        \State \Return model $M_t$, dataset $\mathcal{D}_t$
    \EndIf
\EndFor
\State \Return Model $M_{\texttt{max\_iter}+1}$, dataset $\mathcal{D}_{\texttt{max\_iter}+1}$
\end{algorithmic}
\end{algorithm}

\subsection{Gradient-based Rounding}\label{subsec:round}

After all pruning steps are complete, we backtrack through the remaining network weights to recover the symbolic formula. To further simplify the resulting expression, we apply a rounding step to the constants \( c_i \) within the symbolic formula. 
By \autoref{thm:output-taylor}, the change in the symbolic function \( f \) due to modifying a constant is bounded by a first-order approximation. Specifically, we estimate the effect of replacing each \( c_i \) with a  rounded value \( r_i \) (e.g., rounding to the closest first decimal place) as:
\[
|f(c_i) - f(r_i)| \leq |c_i - r_i| \cdot \max_{s \in [r_i, c_i]} \left| \frac{\partial f}{\partial c_i}(s) \right|.
\]
To avoid computing gradients at multiple points, we approximate the maximum by evaluating the derivative at the rounded value \( r_i \), yielding:
\[
|f(c_i) - f(r_i)| \approx |c_i - r_i| \cdot \left| \frac{\partial f}{\partial c_i}(r_i) \right|.
\]
If this estimated change is smaller than a predefined threshold for all input samples, we replace \( c_i \) with the rounded value \( r_i \), as its influence on the function is negligible.

We apply gradient-based rounding twice during the final stages of formula extraction. First, we perform an initial rounding to zero before coefficient optimization (\autoref{subsec:optim}) to eliminate unnecessary terms, enhance sparsity, and reduce expression complexity, which in turn improves the efficiency of the optimization step. 
After optimization, we apply a second gradient-based rounding pass to check if any constants can be rounded to at most one decimal place, to obtain a cleaner symbolic formula.

\subsection{Coefficient Optimization}\label{subsec:optim}

Following the pruning and rounding steps, the symbolic structure of the expression is fixed, and we proceed to optimize the remaining numerical constants. We frame this as a regression problem and utilize least squares fitting to minimize the discrepancy between the obtained formula and the ground truth by adjusting only the constants within the expression. This refinement ensures that the final symbolic formula achieves a close fit to the underlying data after pruning.

\subsection{Loss Function and Training}\label{subsec:loss}

Since a broad class of analytic expressions in nature can be written as multiplicative, divisive, or exponential functions of the underlying variables, applying a logarithmic transformation makes them easier to approximate using linear or polynomial models. 
This transformation enables the LIES architecture to match the performance achieved in the original input space, while requiring fewer layers—thereby improving training efficiency and reducing the risk of overfitting. We 
apply max normalization to the inputs to accelerate convergence and simplify the reverse-scaling process, ultimately aiding the recovery of interpretable symbolic expressions.

We define the objective function as a sum of four terms, which are detailed below. 
\newline
\textit{$\mathit{1}$. Exponential loss}: 
The principal part of the objective function is the exponential loss, which resembles the Mean Absolute Error (MAE) but operates in the exponential domain, as all inputs have been transformed into the log space. The exponential loss function is defined as:
\begin{equation}
\mathcal{L_{\text{exp}}} = \frac{1}{N} \sum_{i=1}^{N} \exp\left( \left| \hat{y}_i - y_i \right| \right),
\end{equation}
where $\hat{y}$ is the predicted value, ${y}$ is the target value, and $N$ is the batch size.

\textit{$\mathit{2}$. Sparsity-inducing loss}:
As mentioned in \autoref{subsec:sparsity}, promoting interpretability is pivotal in designing symbolic regression models, which can be achieved by inducing sparsity in the network. Sparsity can be enforced through two steps. 
Firstly, as in \eqref{eq:wupdate}, we incorporate an additional term derived from the ADMM framework to encourage weights of the network toward zero:
\begin{equation}
\mathcal{L_{\text{ADMM}}} =\frac{\rho}{2} \|W - Z^{(k)} + U^{(k)}\|_2^2,
\end{equation}
where $W$ is the actual weight, $Z$ is an auxiliary variable to facilitate optimization, $U$ is a dual variable to ensure $W$ and $Z$ remain close, and $\rho$ is a penalty parameter.

Secondly, we include an $\ell_1$ regularization term as a global push towards sparsity to support faster convergence, especially in early training stages. Here, $w_i$ represents each scalar weight parameter in the network:
\begin{equation}
\mathcal{L}_{\ell_1} =\alpha \sum_{i} |w_i|.
\end{equation}

\textit{$\mathit{3}$. Activation function loss}:
Due to their inherent mathematical properties, logarithmic and exponential functions pose some challenges, limiting their suitability to be used as activation functions. The rapid growth of the exponential function and its derivatives at high input values can lead to instability during gradient descent. A similar issue arises with the derivatives of the logarithm function at near-zero inputs. Also, the logarithm is undefined unless the inputs are strictly positive. To handle these issues, we need to modify these two functions while preserving their differentiability during gradient descent. Therefore, we apply a cutoff to these functions and use the following logarithmic and exponential functions for activations:
\begin{center}
\begin{tabular}{cc}
$
\operatorname{L}(x)= \begin{cases}\ln(x) & x>x_{l} \\ \ln(x_l) & \text{ else }\end{cases} 
$ &
$
\operatorname{E}(x)= \begin{cases}\exp(x) & x<x_{e} \\ \exp(x_e) & \text{ else }\end{cases} 
$
\end{tabular}
\end{center}
where we select $x_{l}=5e^{-3}$ and $x_{e}=4$.

To encourage the inputs of each neuron to remain within the valid domain of the logarithmic and exponential functions, we add a masked auxiliary loss that penalizes inputs falling below the predefined cutoff values for both functions:
\begin{equation}
\begin{aligned}
\mathcal{L}_{\text{mask}} = &\sum_{i} \mathds{1}_{\{x_{i,\log} < x_l\}} \cdot \left| x_l - x_{i,\log} \right| \\
&+ \sum_{i} \mathds{1}_{\{x_{i,\exp} > x_e\}} \cdot \left| x_e - x_{i,\exp} \right|
\end{aligned}
\end{equation}
where $x_{i,\log}$ and $x_{i,\exp}$ denote the inputs of sample $i$ to the logarithmic and exponential activation functions, respectively. 

Finally, the total loss function to be optimized in the proposed framework is:
\begin{equation}
    \mathcal{L}_{\text{total}} = \mathcal{L}_{\text{exp}} + \mathcal{L}_{\text{ADMM}} + \mathcal{L}_{\ell_1} + \mathcal{L}_{\text{mask}}
\end{equation}

\section{Experiments and Results}\label{sec:experiments}

\subsection{Dataset and Model}\label{subsec:data}
We evaluate the proposed LIES architecture on the AI Feynman dataset~\cite{aifeynman}, a well-established benchmark for symbolic regression consisting of 100 physics-based equations derived from the Feynman Lectures. Each task provides a numerically generated dataset along with its ground truth symbolic expression, incorporating operations such as polynomials, trigonometric functions, exponentials, and logarithms. 
The equations involve between 1 and 7 input variables. This dataset is also included in SRBench~\cite{lacava}, a recent standardized benchmark suite for symbolic regression, which we use to compare our method against established baseline approaches. 
To reduce training time, we use only 10\% of the available data for each task during training (100k samples). Note that baseline methods are evaluated using the full dataset, which may give them an advantage in terms of data availability. As mentioned in \autoref{subsec:loss}, the data is first transferred to the log space and then max normalized.

If the number of input variables in a formula is $n$, the LIES network is configured with $n+1$ LIES layers. 
We focus on the formulae with four or fewer input variables. 
Higher-dimensional formulae require deeper networks with more parameters, which makes pruning more challenging and hinders the recovery of sparse symbolic expressions~\cite{interactive2025}. 
We also exclude formulae involving trigonometric functions that have negative inputs, as our use of log-space transformations can lead to instability in such cases. Please refer to~\autoref{sec:discuss} for a detailed discussion. Therefore, we conduct our experiments on 61 formulae in the AI Feynman dataset. Each experiment is conducted across three independent trials to account for the stochastic nature of training and ensure the consistency and robustness of the recovered expressions. Each of the weak ADMM training parts is run for 20 epochs and the strong ADMM is run for 30 epochs. The whole pipeline takes around 10 minutes to run on a single NVIDIA RTX A5000 GPU for a single trial. The time complexity of each ADMM training phase is $\mathcal{O}(E \cdot B \cdot P)$, where $E$ denotes the number of training epochs, $B$ is the number of mini-batches per epoch, and $P$ is the total number of trainable parameters in the model. During the oversampling phase, this training process is repeated up to a maximum of \texttt{max\_iter} times, resulting in an additional runtime of at most $\mathcal{O}(E \cdot B \cdot P \cdot \texttt{max\_iter})$, which is additive to the overall complexity. The gradient pruning step has a time complexity of $\mathcal{O}(P)$, where $P$ is the total number of trainable parameters. The time complexity of gradient rounding is $\mathcal{O}(C \cdot N)$, where $C$ is the number of constants in the expression and $N$ is the number of data points. This reflects the cost of evaluating gradient expressions across the dataset to determine which constants can be safely rounded.
The coefficient optimization step has a time complexity of $\mathcal{O}(I \cdot C \cdot N)$, where $I$ is the number of optimization iterations, $C$ is the number of numerical constants being optimized, and $N$ is the number of data points. Therefore, the overall time complexity of the proposed pipeline is $\mathcal{O}\left( E \cdot B \cdot P \cdot (\texttt{max\_iter} + 3) + I \cdot C \cdot N \right)$. The ADMM-based training and oversampling dominate runtime, while coefficient optimization and the two rounding stages contribute linearly with respect to the number of constants and dataset size.

We use the RMSprop optimization algorithm with a learning rate of $1.5 \times 10^{-2}$. In the weak ADMM training phases, we set the penalty parameter $\rho = 0.5$ and the regularization coefficient $\lambda = 5 \times 10^{-4}$. During the strong ADMM training phase, we use a smaller penalty parameter $\rho = 0.005$ and adapt the regularization coefficient as $\lambda = 5 \times 10^{-(n-1)}$, where $n$ is the number of input variables in the target expression. For the gradient-based pruning step, we apply a sensitivity threshold of $0.01$ and in the gradient-based rounding step, we use a threshold of $0.1$ to simplify small coefficients while maintaining functional fidelity.
 All our code is publicly available\footnote{\url{https://github.com/MansoorehMontazerin/LIES}}.

\subsection{Evaluation Metrics}\label{subsec:metrics}

We evaluate our method using the $R^2$ score
\[
R^2 = 1 - \frac{\sum_{i=1}^{N} (y_i - \hat{y}_i)^2}{\sum_{i=1}^{N} (y_i - \bar{y})^2}, 
\]
and the symbolic solution rate (SSR):
\[
\text{SSR} = \frac{1}{N} \sum_{i=1}^{N} \mathds{1}_{\{\texttt{sym\_solution}_i = \text{True}\}},
\]
where $\hat{y}_i$ is the predicted value, $y_i$ is the target value and $\bar{y}$ is the mean over all target values. 
SSR measures how frequently the model successfully recovers an expression that is symbolically equivalent to the ground-truth formula. By symbolically equivalent, we mean that the discovered expression is mathematically identical to the ground-truth formula if we add or multiply a constant to it. 
For each trial, a binary flag (\texttt{sym\_solution}) indicates whether the recovered expression is symbolically correct. SSR is computed as the proportion of such successful cases across all trials and all formulae ($N$), providing a strict assessment of the model’s ability to recover exact symbolic representations, rather than just numerically accurate approximations.


\subsection{Performance Comparison with Baseline Methods}\label{subsec:comparison}

We compare LIES with the SRBench methods in terms of the frequency with which the model achieves a solution with $R^2>0.99$ and the SSR. For a fair comparison, we evaluate all baseline methods on the same subset of 61 formulae used in our experiments. \autoref{fig:r2} reports the mean and standard deviation of the accuracy ($R^2>0.99$) across trials and \autoref{fig:ssr} presents the results for the SSR. The $R^2$ score is computed by evaluating the final recovered equation on the input data and comparing its predictions to the ground-truth outputs. Notably, LIES consistently achieves a higher SSR compared to both DL-based methods, such as DSR, and GP-based methods, including gplearn and GP-GOMEA. This highlights the model’s ability to recover symbolically correct expressions rather than simply overfitting to data. On the other hand, LIES shows moderately lower accuracy in some cases. This discrepancy is primarily due to minor deviations in the recovered constant values—while the overall symbolic structure may be correct, small offsets in constants can reduce the test $R^2$ below the threshold, particularly for expressions where constants significantly affect output magnitude.

\begin{figure}[htbp]
\centering
  \includegraphics[width=0.85\linewidth]{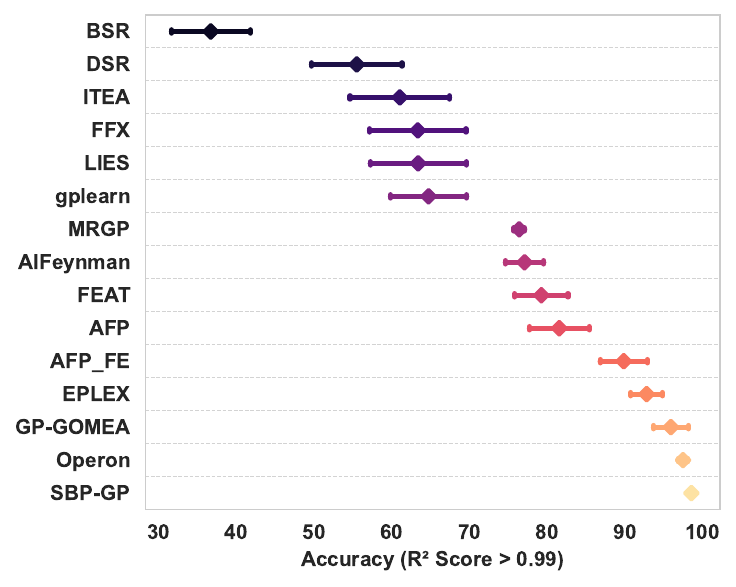}
  \caption{Mean and standard deviation of the accuracy ($R^2>0.99$) across trials in LIES and the SRBench models on 61 equations from the Feynman dataset.}
  \label{fig:r2}
\end{figure}

\begin{figure}[htbp]
\centering
  \includegraphics[width=0.85\linewidth]{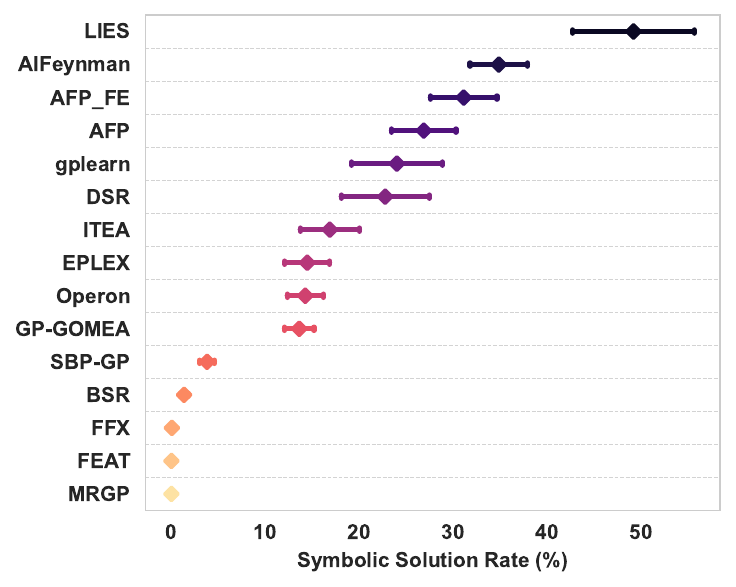}
  \caption{Mean and standard deviation of the SSR across trials in LIES and the SRBench models on 61 equations from the Feynman dataset.}
  \label{fig:ssr}
\end{figure}

\begin{table*}[htbp]
\centering
\caption{Performance comparison of LIES with different configurations within a fixed deadline. All values are reported as percentages averaged across trials.}
\label{tab:ablation}
\begin{tabular}{c c c c c}
\hline
Components & Sym Solution True  & Sym Solution False  & Out-of-Time Error  &  $R^2 > 0.99$ \\
\hline
\textbf{Main Pipeline} & 98.9& 0& 1.1& 88.9\\
w/o Oversampling & 70 & 20 & 10 & 66.7 \\
w/o Gradient Pruning & 0 & 0 & 100 & N/A \\
w/o Gradient Rounding to Zero   & 34.4 & 54.5 & 11.1 & 27.8 \\
w/o Optimization & 20 & 76.7 & 3.3 & 23.3 \\
\hline
\end{tabular}
\end{table*}

\subsection{Ablation Study}\label{subsec:ablation}

To evaluate the contribution of each component in our symbolic regression pipeline, we conduct an ablation study. By systematically removing individual modules and observing the resulting performance degradation, we demonstrate that each part of the pipeline plays a critical role in ensuring the efficiency and sparsity of the final solution. The key components of our symbolic regression pipeline, which are the focus of our study, are as follows: \textit{1. Oversampling, 2. Gradient Pruning, 3. Gradient Rounding to Zero, and 4. Constant Optimization}. 
We pick the formulae where at least one trial yielded a correct symbolic solution under the full pipeline configuration (30 out of 61), so that we could evaluate the importance of each component. We consider three outcomes: (a) the recovered symbolic solution is correct, (b) the recovered solution is incorrect, and (c) an out-of-time error occurs (defined as the program exceeding five times the typical runtime when all modules are active). 
For cases where the solution is either correct or incorrect, we additionally report the frequency of achieving a test set $R^2 > 0.99$. \autoref{tab:ablation} summarizes the results of the ablation study described above. 
Note that the ``Symbolic Solution True" column in \autoref{tab:ablation} corresponds to the SSR metric discussed earlier. Unlike the earlier SSR metric, the table now distinguishes between failures due to incorrect symbolic solutions (``False") and those caused by exceeding the runtime limit (``Out-of-Time Error"), offering a more detailed breakdown of the pipeline performance.
In what follows, we analyze the results of removing each of the key components listed in \autoref{tab:ablation} from the pipeline, highlighting their individual contributions to the overall performance.

\textbf{$\mathbf{1}$. Oversampling}  The oversampling step targets regions of the input space where the model exhibits the largest prediction errors. By providing additional training samples in these regions, oversampling helps the model refine its predictions, leading to more accurate symbolic solutions and higher $R^2$ scores. As shown in \autoref{tab:ablation}, removing the oversampling module results in a 20\% drop in SSR and a 22\% reduction in the frequency of achieving $R^2>0.99$.

\textbf{$\mathbf{2}$. Gradient Pruning}  The gradient pruning module serves as a complementary pruning method to ADMM-based pruning. While ADMM encourages sparsity by penalizing the magnitude of weights, gradient pruning eliminates additional weights by evaluating the importance of their gradients. 
As shown in \autoref{tab:ablation}, removing this component significantly degrades performance and leads to out-of-time error for all the formulae. This is because, without gradient pruning, many small and unnecessary weights remain in the network, leading to excessively large and dense formulae populated with negligible coefficients, which compromise the sparsity and clarity of the recovered expressions. Moreover, this directly impacts the efficiency of the gradient-based rounding to zero step, as the presence of many non-zero but unimportant weights increases the computational cost of symbolic differentiation, resulting in frequent out-of-time errors. We report ``N/A" (not applicable) for the $R^2 > 0.99$ metric since the out-of-time error hinders computation of $R^2$ scores.

\textbf{$\mathbf{3}$. Gradient Rounding to Zero}  This module implements a specialized form of gradient-based rounding designed to eliminate extremely small coefficients from the recovered symbolic expression. Its primary purpose is to prevent large computational slowdowns in the subsequent coefficient optimization step, where numerous insignificant coefficients can increase the cost of the regression-based fitting procedure. This rounding method evaluates the impact of removing a coefficient by computing the product of the symbolic derivative of the expression with respect to that coefficient (evaluated at the rounded value) and the difference between the original and rounded coefficient. If this product is sufficiently small, the coefficient is set to zero. As shown in \autoref{tab:ablation}, removing this component severely degrades performance, resulting in a nearly 60\% drop in both SSR and accuracy.

\textbf{$\mathbf{4}$. Coefficient Optimization}  The coefficient optimization module is designed to refine the numerical coefficients of the symbolic expression after several pruning steps have been applied. While the raw formula obtained from the model may initially fit the data well, pruning can introduce structural changes that negatively affect its alignment with the data. This step adjusts the remaining coefficients to better align the expression with the data, effectively serving as a lightweight fine-tuning phase. As shown in \autoref{tab:ablation}, removing this component leads to a substantial performance drop, with the SSR decreasing by 80\% and the accuracy by around 65\%.

\paragraph*{Other Pruning Strategies} Since LIES Network has activations with different ranges, a small weight does not necessarily imply a good candidate for pruning. 
We considered using reweighted $\ell_1$ regularization~\cite{reweight2008} with ADMM, which is expected to accommodate such differences better. However, it resulted in poor sparsity and unstable convergence. Our pipeline is already able to handle different ranges through gradient-based pruning that approximates the impact of setting weights to zero instead of pruning all weights smaller than a threshold. 

\section{Discussion and Future Work}\label{sec:discuss}

We proposed LIES, an SR framework that trains a fixed neural network with interpretable activations and applies structured pruning to extract accurate, compact, and human-readable formulae. We evaluated LIES on 61 expressions from the AI Feynman dataset, demonstrating strong performance both symbolically and numerically, and we conducted ablation studies to assess the impact of each component in the pipeline.

The main limitations of our current approach lie in handling more complex formulae in terms of the number of variables and the use of trigonometric functions. Below, we discuss potential directions to address these limitations in the future.

\subsubsection{Dealing with Trigonometric Functions}

To explore the potential of LIES for trigonometric expressions, we conducted a targeted experiment using a synthetic dataset with inputs $x$ and outputs $\cos(x)$. 
The framework successfully represented $\cos(x)$ as $\sin(\pi/2 - x)$, confirming the utility of the sine activation function for handling such expressions. 
However, long formulae including trigonometric functions require deeper LIES networks. Creating multiplications in such deep networks requires going through logarithms. The logarithm of negative outputs of the trigonometric function is currently cut off in our training to ensure real gradients. This makes it difficult to learn formulae involving the multiplication of trigonometric functions with other functions. In future work, we will explore training the network in the complex domain to support the logarithm of negative quantities, and thus support longer formulae with trigonometric functions.
\subsubsection{Further Improvements}

Handling formulae with a large number of variables remains challenging for the LIES network, often requiring more effective pruning and sparsification strategies. Particularly, formulae requiring functions of the sum of products (e.g., $\sqrt{x^2 + y^2 + z^2}$) require deep networks, making them hard to prune and learn. We observed that such formulae are highly sensitive to parameter tuning and prone to unstable training. In future work, we aim to develop improved pruning techniques, investigate architecture search~\cite{darts}, and enhance optimization strategies to enable the inclusion of larger and more complex symbolic formulae.

\section{Conclusion}

In this work, we presented LIES, a structured neural framework for SR that uses interpretable activation functions and targeted pruning strategies to produce compact, accurate expressions. Across a range of benchmark tasks, LIES achieved high $R^2$ scores and effectively recovered ground-truth formulae with minimal complexity. 
Our experiments, including detailed ablation studies, demonstrate the contribution of each component in the pipeline. The multi-stage pruning process, supported by Taylor-based approximations, consistently led to sparser and more meaningful formulae.




\bibliographystyle{IEEEtran}
\bibliography{references}

\end{document}